\documentclass[journal]{IEEEtran}  

\IEEEoverridecommandlockouts                             
\makeatletter
\let\NAT@parse\undefined
\makeatother
\usepackage[square, sort&compress, comma, numbers]{natbib}

\usepackage{graphicx}
\usepackage{subfigure}
\graphicspath{{figures/}{./}} 
\usepackage[latin1]{inputenc}
\usepackage{fancyhdr}    
\usepackage{textcomp}  
\usepackage{amsmath} 
\usepackage{amssymb}
\usepackage{dsfont}
\usepackage{psfrag}
\usepackage{stfloats}
\usepackage{color}
\usepackage{multicol}
\usepackage{algorithmic}
\usepackage{algorithm}
\usepackage{array}
\usepackage{bm}
\usepackage{bbm}
\usepackage{paralist}{}
\usepackage{wasysym}
\usepackage[normalem]{ulem}

\usepackage[bookmarks=true]{hyperref}
\usepackage[hyphenbreaks]{breakurl} 

\usepackage{amsthm}
\newtheorem{theorem}{Theorem}
\newtheorem{lemma}{Lemma}
\newtheorem{proposition}{Proposition}
\newtheorem{remark}{Remark}
\newtheorem{problem}{Problem}
\newtheorem{definition}{Definition}

\newcommand{\mdots}{\ifmmode\mathinner{\ldotp\kern-0.2em\ldotp\kern-0.2em\ldotp}\else.\kern-0.13em.\kern-0.13em.\fi}

\begin{document}

\title{Resilient Assignment Using Redundant Robots on Transport Networks with Uncertain Travel Time}

\author{
    \IEEEauthorblockN{Amanda Prorok\\} 
    \IEEEauthorblockA{Department of Computer Science and Technology\\ 
    University of Cambridge, UK\\ 
    {\tt\small asp45@cam.ac.uk}}
}

\maketitle

\begin{abstract}
This paper considers the problem of assigning multiple mobile robots to goals on transport networks with uncertain information about travel times. Our aim is to produce optimal assignments, such that the average waiting time at destinations is minimized. Since noisy travel time estimates result in sub-optimal assignments, we propose a method that offers resilience to uncertainty by making use of redundant robots. 
However, solving the redundant assignment problem optimally is strongly NP-hard. Hence, we exploit structural properties of our mathematical problem formulation to propose a polynomial-time, near-optimal solution. We demonstrate that our problem can be reduced to minimizing a supermodular cost function subject to a matroid constraint. This allows us to develop a greedy algorithm, for which we derive sub-optimality bounds. We demonstrate the effectiveness of our approach with simulations on transport networks, where uncertain edge costs and uncertain node positions lead to noisy travel time estimates.
Comparisons to benchmark algorithms show that our method performs near-optimally and significantly better than non-redundant assignment.
Finally, our findings include results on the benefit of diversity and complementarity in redundant robot coalitions; these insights contribute towards providing resilience to uncertainty through targeted robot team compositions.
\end{abstract}


\renewcommand\abstractname{Note to Practitioners}
\begin{abstract}
This paper is motivated by the problem of assigning mobile robots to goals when travel times from robot origins to goal locations are uncertain. Existing robust assignment methods deal with uncertainty by minimizing risk, or by pre-defining acceptable risk thresholds. In this work, we propose a complementary method that offers resilience to uncertainty by making use of robot redundancy. In other words, we assign more robots than necessary to a given goal, in the expectation that one of the redundant robots will reach the goal faster (than the originally assigned robot). However, solving this redundant assignment problem is computationally intractable for large systems. By characterizing the mathematical problem, we show how the redundant assignment problem can be solved efficiently. 
We apply our assignment algorithm to transport network problems to reduce average waiting times at goal locations, when travel times from vehicle origins to destinations are uncertain. Our results show that exploiting robot redundancy is an effective approach to reducing waiting times. In this work, we build on the premise that time is the primary commodity, and we do not model the additional cost of utilizing redundant robots. Although we do provide an additional complementary problem definition that minimizes the number of robots used, future work should more explicitly address the trade-off between the cost of providing redundancy (e.g., travel costs, robot costs) and performance gains. 
%
\end{abstract}



\section{Introduction}
The optimal assignment of mobile robots to tasks at different locations is a fundamental problem that has many applications. In particular, the large-scale deployment of robots has the potential of transforming the industries of transport and logistics.
To orchestrate the coordination of the robots, centralized communication architectures have become the norm in various instances;
representative applications include product pickup and delivery~\cite{grippa:2017}, item retrieval in warehouses~\cite{enright:2011}, and mobility-on-demand services~\cite{spieser2016shared,pavone2012robotic}. In general, a solution to this problem can be computed by a centralized unit that collects all robot-to-task assignment costs (e.g., expected travel times) to determine the optimal assignment (e.g., by running the Hungarian algorithm). However, the optimality of this assignment hinges on the accuracy of the assignment cost estimates. Despite our best efforts to model any uncertainties, discrepancies between model assumptions and real-life dynamics may arise. For example, in transport scenarios, a robot may encounter an unexpectedly blocked path, and consequently takes significantly longer to reach its destination than anticipated. 
These discrepancies cause degradations in the system's overall performance, and can lead to cascading effects.

\begin{figure}[tb]
\centering
\psfrag{t}[cc][][0.7][90]{Expected Travel Time [s]}
\psfrag{r}[cc][][0.7]{Random}
\psfrag{u}[cc][][0.7]{Repeated Hung.}
\psfrag{e}[cc][][0.7]{Greedy}
\psfrag{C}[cc][][0.7][90]{Correlation}
\subfigure{\includegraphics[height=5.5cm]{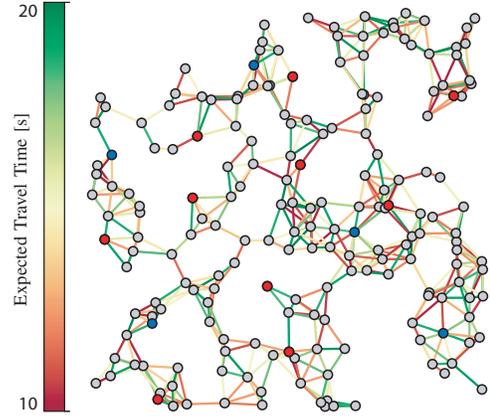}\label{fig:graph_instance_a}}
\caption{An instance of a random transport network. Robots are initially located at hubs (red nodes). Task locations (goals) are marked by nodes colored in blue; The travel time along the edges of this transport network is uncertain, and is modeled by a probability distribution; the colormap indicates the expected travel time.
\label{fig:graph_instance}}
\end{figure}

We are interested in applications that require a fast arrival of robots at their destinations. Due to uncertain robot locations or uncertain path conditions, the knowledge about expected travel times may be imprecise or incomplete. This compounds the difficulty of determining an optimal robot-to-goal assignment. Although assignments under random costs have gained a considerable amount of attention~\cite{krokhmal:2009,Nam:2017jwa,Nam:2015hz}, the focus has primarily been on providing analyses of the performance under noisy conditions. 
In this work, we propose a complementary method that provides resilience to noisy travel time estimates by making use of \emph{robot redundancy}. In other words, the core idea of our work is to exploit redundancy to counter uncertainty and redeem performance. 
Although the idea of engineering robust systems with redundant resources is not new in a broad sense~\cite{kulturel:2003,ghare:1969}, we are the first to consider redundant mechanisms for the problem of mobile robot assignment under uncertainty, with arbitrary and potentially correlated probability distributions. We believe that providing redundant robots will be a fundamental design feature for systems where time is the primary asset (e.g., rescue scenarios), and where an over-provisioning with respect to the number of robots is a minor concern.

The premise of this work is that time is the most valuable asset in the system. Our idea is led by a {\emph{`first-come, first-to-serve'}} principle, by which only the fastest robot to reach a task actually services it. Redundancy allows us to reduce the waiting time at the goal locations: when multiple robots travel to the same destination, only the travel time of the fastest robot counts. Consequently, this paper presents a novel and efficient method by which robots are redundantly matched to individual goals, such that the average waiting time over all goal locations is minimized.

Our problem belongs to the class of Single-Task Robots, Multi-Robot Tasks (ST-MR)~\cite{Gerkey:2004il}, which considers the assignment of groups of robots that have a combined utility for any given task. This task utility can be uniquely defined for a given robot group and task, and is not necessarily linear over the individual robot utilities. The aim is to split the set of robots to form task-specific coalitions, such that the average task utilities over all coalition-to-task assignments is maximized (or, in our case, the average cost minimized). Formally, this can be cast as a set-partitioning problem, which considers a set of robots that needs to be partitioned into a family of subsets with maximum utility over the set of tasks that they are assigned to. The set-partitioning problem is strongly NP-hard~\cite{garey:1978}. Research in this domain has generated heuristic solutions as well as methods that facilitate the combinatorial search~\cite{atamturk:1996,cho:2015,umetani:2015}. 
However, if we pay special attention to the objective function, the problem may reveal additional structure that can be exploited to find near-optimal solutions. In particular, for utility functions that satisfy a property of \emph{diminishing returns}, otherwise known as \emph{submodularity}, near-optimal approximations can be found~\cite{fujishige:2005}. In one of the key contributions of this work, we show that our approach provides diminishing returns in the number of redundant robots assigned to each goal, yielding a \emph{supermodular} cost function. Thus, we pose a set optimization problem under cardinality constraints, and employ a method based on supermodular cost minimization under a matroid constraint. This framework leads to a polynomial-time algorithm that has a provable bound on the difference between the optimal waiting time and the waiting time resulting from our algorithm.


\section{Related work}
\label{sec:related_work}
Our problem is related to the general class of submodular welfare problems~\cite{lehmann:2001,vondrak:2008}, which stems from the field of combinatorial auctions. The welfare problem considers a set of items and a set of players, and seeks a partition of the items into disjoint sets assigned to players in order maximize the total welfare over the players. The \emph{welfare} is equivalent to the sum of utilities over all sets. The utility functions satisfy the property of diminishing returns, and hence, the problem can be formulated as one of submodular maximization.
In contrast to our work, the welfare problem does not prescribe any explicit form for the submodular function, nor does it explicitly consider uncertainty. Instead, it assumes a \emph{value oracle model}, which is a black-box that returns the utility for any given set. 
In this sense, our problem is a specialization of the general submodular welfare problem. We consider a specific objective function, where the submodularity (or, in our case, the supermodularity) arises due to the redundancy of assigned robots with uncertain travel times.
Furthermore, we also provide a specialization of the matroid constraint, which, in our case, relates to the constraint on maximum possible robot deployment sizes.  

Another related body of work deals with the weapon-target assignment problem~\cite{ahuja:2007}, which considers the assignment of weapons to targets so that the total expected survival value of the targets is minimized. Similar to our problem, this problem considers the assignment of a redundant number of items (i.e., weapons) to a single goal (i.e., target), where assignment costs are uncertain (i.e., probability of target survival). 
In contrast to our work, however, the weapon-target assignment problem only applies to binomial distributions that model the outcome (i.e, survival) of each target as a Bernoulli random variable. Our algorithm is capable of dealing with arbitrary (and potentially correlated) probability distributions that describe the assignment costs.
In that sense, our problem is a generalization of the weapon-target assignment problem.

The work in~\cite{golovin:2011} develops an adaptive optimization approach under partial observability of the cost function. The authors introduce the concepts of expected marginal gain and adaptive submodularity, which they leverage to develop an adaptive greedy algorithm. Although these theoretical results are amenable to the types of problems we are interested in (i.e., assignment under uncertainty), they are founded on an orthogonal approach, which consists of adapting cost estimates online, instead of hedging against uncertainty by redundancy a-priori, as we do.

Submodular optimization for combinatorial problems has has also gained considerable traction in the domain of multi-robot systems. Applications include coordinated robot routing for environmental monitoring~\cite{singh:2009}, leader selection in leader-follower systems~\cite{clark:2014}, sensor scheduling for localization~\cite{singh:2017}, path planning for orienteering missions~\cite{jorgensen:2017}, and constrained task allocation~\cite{williams:2017}. Typically, the aforementioned studies develop explicit objective functions and constraints that are specific to the considered problem domains. The authors then go about proving the submodularity property, deriving sub-optimality bounds, and devising the appropriate assignment algorithms. This general methodology is similar to the one presented in our paper.

In summary, the particularity of our work lies in the specificity of our objective function for redundant robot assignments under uncertain travel times on random transport networks. Robot redundancy plays a key role in our theoretical developments; our submodularity property is tightly coupled to travel applications, and hinges on the `first-come, first-to-serve' principle.
Finally, our findings include results on the benefit of diversity and complementarity in redundant robot coalitions; these are unprecedented insights within this context, and contribute towards providing resilience to uncertainty through targeted robot team compositions.
 
\emph{Contributions.}
The main contribution of this work is a supermodular optimization framework that selects redundant robot-to-goal matchings in order to minimize the average waiting time at the task locations. Additionally, we pose a variant of this optimization problem such that the number of assigned robots is minimized, subject to a maximum allowable cost budget. Our framework is underpinned by the novel insight that the deployment of redundant robots under uncertain travel times is supermodular in the number of redundant robots. This insight is derived from a \emph{first-come, first-to-serve} principle, which is formally introduced in this paper through an aggregate function that considers only the \emph{minimum} travel time among the assigned robot coalition. 
Furthermore, we formalize the cardinality and assignment constraints by a matroid. These combined results (optimization of a supermodular function under a matroid constraint) allow us to derive sub-optimality bounds on the performance of our method.

Our second contribution is the development of a dynamic programming algorithm that reduces the number of calls to the objective function, and allows us to compute task utilities efficiently in polynomial time. This algorithm leverages the concept of aggregate functions over robot coalitions, which we formally introduce in this work.
A key consideration is that the performance at each goal is measured by an aggregate cost function that considers the joint performance of all assigned robots at that goal. 

Finally, we apply the redundant assignment solution to random instances of transport networks in planar space. In particular, we consider two causes for uncertain travel time: uncertainty in the robot origin positions, and uncertainty in the road travel time of the transport network (which was considered in our prior work~\cite{prorok:2019DARS}). We compare our method with an optimal, exhaustive combinatorial search algorithm to show that our method performs very close to the optimum. Our approach is also compared to several benchmark algorithms, and is evaluated under varying noise conditions. We show that in all considered cases, redundant assignment effectively increases resilience to uncertainty by reducing the performance gap to an ideal, non-noisy system.

\section{Problem Statement}

We consider a system composed of $M$ goals and $N$ available robots. 
Our problem considers the assignment of robots to goals via a path, where, for each robot-to-goal assignment, multiple possible paths with random costs exist. Without loss of generality, we assume that $K$ possible paths exist between each robot and each goal.
We seek to find a minimum cost matching, such that all goals are covered, any goal may be assigned multiple robots, and each robot uses one of $K$ paths to reach its assigned goal.

The mathematical model and algorithms presented in this paper are general, and can accommodate uncertainties defined on arbitrary task assignment problems. In order to provide specific examples and their implementations, we apply the assignment solution to a transport scenario (in planar space), where imprecise information about robot positions or road travel times leads to uncertainty in the robot travel time estimates (and hence, assignment cost uncertainty).
In specific, we consider two problem variants (see Problem~\ref{prob:problem1} and Problem~\ref{prob:problem2}). The first problem aims to minimize the average assignment cost over all goals, while respecting a limit on the maximum deployment size $N_{\mathrm{d}}$. The second problem aims to minimize the number of deployed robots, while respecting a limit on the maximum admissible average assignment cost, $\xi$. 
For each of these problems, we consider that assignment cost uncertainty arises due to two main causes: \emph{(i)} robot position uncertainty, and \emph{(ii)}, road travel-time uncertainty. Sections~\ref{sec:application_edges} and~\ref{sec:application_nodes} elaborate these scenarios.

\subsection{Redundant Assignment with Random Costs} 
Consider a graph $\mathcal{B}= (\mathcal{U},\mathcal{F}, \mathcal{C})$.
The set of vertices $\mathcal{U}$ is partitioned into two subsets $\mathcal{U}_r$ and $\mathcal{U}_g$, such that $\mathcal{U}_r = \bigcup_i r_i,~i=1,\ldots,N$ contains all robot nodes, $\mathcal{U}_g = \bigcup_j g_j, ~j=1,\ldots, M$ contains all goal nodes, and $\mathcal{U} = \mathcal{U}_r \, \cup \, \mathcal{U}_g$, $\mathcal{U}_r \, \cap \, \mathcal{U}_g = \emptyset$. We define a fixed  number $K$ of possible path routes that lead any robot to any goal. The edge set $\mathcal{F} = \{(i,j,k)| i \in \mathcal{U}_r, j \in \mathcal{U}_g, k \in 1,\ldots,K, \forall \, i,j,k\}$ is complete, meaning that any robot can reach any goal node, and that up to $K$ possible path choices exist for any pair $(i,j)$. We note that for $K=1$, the problem becomes equivalent to a bi-partite graph matching problem with uncertain costs.
The tuple $(i,j,k)$ indicates that robot $i$ is assigned to goal $j$ through path $k$.
Since the travel time for a robot to reach its goal is uncertain, we represent the weight of each edge $(i,j,k)$ by a random variable $C_{ijk} \in \mathcal{C}$, where $\mathcal{C}$ is a set of random variables. The set $\mathcal{C}$ has a joint distribution $\mathcal{D}$.
Hence, $C_{ijk}$ can be arbitrarily defined for any edge; in particular, edge costs may be correlated, and we do not make use of \emph{i.i.d.} assumptions.

We consider the case where an initial non-redundant assignment has been made. The initial assignment is $\mathcal{O} \subset \mathcal{F}$, such that $\forall j |\{i | (i,j,k) \in \mathcal{O} \}| = 1$, $\forall i |\{j | (i,j,k) \in \mathcal{O} \}| \leq 1$, and $\forall i, j |\{k | (i,j,k) \in \mathcal{O} \}| \leq 1$. In other words, $\mathcal{O}$ covers every goal with one robot, and any robot is assigned to at most one goal through one given {path}. 
Given an initial assignment, the aim of this paper is to find an optimal set of assignments $\mathcal{A}^\star \subset \mathcal{F} \setminus \mathcal{O}$ for the remaining $N_{\mathrm{d}} - M$ robots~\footnote{Without $\mathcal{O}$, any solution that is smaller in size than $M$ would lead to an infinite waiting time, and hence, the objective function looses its supermodular property. The assumption that we already have an initial assignment is necessary, for the developments that follow.}. From here on, we denote $\mathcal{F} \setminus \mathcal{O}$ by $\mathcal{F}_\mathcal{O}$.
Fig.~\ref{fig:matching} illustrates a simple robot-to-goal matching.
\begin{figure}[tb]
\centering
\psfrag{A}[cc][][0.8]{$\mathcal{A}$}
\psfrag{O}[cc][][0.8]{$\mathcal{O}$}
\psfrag{i}[cc][][0.8]{$i$}
\psfrag{j}[cc][][0.8]{$j$}
\psfrag{1}[cc][][0.8]{$r_1$}
\psfrag{2}[cc][][0.8]{$r_2$}
\psfrag{3}[cc][][0.8]{$r_3$}
\psfrag{4}[cc][][0.8]{$r_4$}
\psfrag{5}[cc][][0.8]{$r_5$}
\psfrag{6}[cc][][0.8]{$r_6$}
\psfrag{7}[cc][][0.8]{$g_1$}
\psfrag{8}[cc][][0.8]{$g_2$}
\psfrag{a}[cc][][0.8]{$\mathcal{U}_r$}
\psfrag{b}[cc][][0.8]{$\mathcal{U}_g$}
\psfrag{c}[cc][][0.8]{${C}_{621}$}
\psfrag{D}[cc][][0.8]{${C}_{622}$}
\includegraphics[width=0.54\columnwidth]{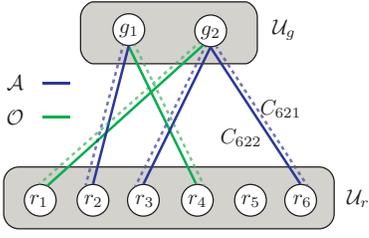}
\caption{Sketch of redundant robot-to-goal matching, for a total of $N=6$ available robots, $M=2$ goals, $K=2$ path options, and a deployment size of $N_{\mathrm{d}}=5$. Edges in $\mathcal{O}$ represent the initial assignment, and edges in $\mathcal{A}$ represent the redundant assignment. The random variable $C_{ijk}$ represents the uncertain travel time it takes robot $i$ to reach goal $j$ via path $k$.
\label{fig:matching}}
\end{figure}

The main novelty of our approach is the use of redundant robots to help counter the adverse effect of uncertainty. If the system admits a sufficiently large number of robots (i.e., $N_{\mathrm{d}} - M > 0$), we can assign multiple robots to the same goal, while still ensuring that all goals are assigned at least one robot. A key consideration is that the performance at each goal is measured by an \emph{aggregate cost} function that considers the joint performance of all assigned robots at that goal.

\begin{definition}[Aggregate Cost] 
\label{def:aggregate}
We define an aggregate function $\Lambda: 2^{\mathcal{F}} \mapsto \mathbb{R}$ that operates over the set of edges incident to a given node, and returns a scalar that represents the aggregate cost over the weights of these edges. If $I_j(\mathcal{A})$ is the set of incident edges to node $j$ in the set of edges $\mathcal{A}$, and is equal to $\{(i,j,k) | \forall i, k \,\mathrm{with}\, (i,j,k) \in \mathcal{A}\}$, then we can write the aggregate cost for goal $j$ as 
\begin{equation}
\Lambda(I_j(\mathcal{A})). 
\end{equation}
\end{definition}
In the following, we also refer to the set of all assigned robots to one goal as a \emph{robot coalition}, following the terminology originally introduced in~\cite{banerjee:2001} and later adapted to the robot domain in~\cite{ouimet:2013}.

\subsection{Optimization Problems}
The definitions above allow us to formulate our objective function. For a given set $\mathcal{A}$ of edges that define robot to goal assignments, we wish to measure the average aggregate cost over all goals, in expectation over the random assignment costs:
\begin{equation} \label{eq:objective_old}
J_{\mathcal{O}}(\mathcal{A}) = \frac{1}{M} \, \sum_{j=1}^{M} ~\mathop{\mathbb{E}}\limits_{\mathcal{C}}\,\left[ \Lambda(I_j(\mathcal{A \cup \mathcal{O}}))  \right].
\end{equation}
We note that when no redundant robots are deployed, the assignment is reduced to the set $\mathcal{O}$, for which the performance is measured as
\begin{equation} \label{eq:J0}
J_0 = J_{\mathcal{O}}(\emptyset) = \frac{1}{M} \, \sum_{j=1}^{M} ~\mathop{\mathbb{E}}\limits_{\mathcal{C}}\,[C_{ijk} | (i,j,k) \in \mathcal{O}]. 
\end{equation}

For practical reasons, we further specify our problem to cap any robot deployment at a maximum number of assignable robots (constrained by $N_{\mathrm{d}}$, with $M \leq N_{\mathrm{d}} \leq N$). This consideration turns out to be particularly relevant for applications that run continuous assignments with fluctuating task demand, and allows us to keep some robots in reserve for future deployments. Furthermore, it allows operators to limit the cost of redundant robot deployments, for example, by monitoring and capping energy consumption. Formally, we capture this additional constraint by a \emph{matroid}~\cite{oxley2006matroid} (see also Def.~\ref{def:matroid} and Sec.~\ref{sec:matroid}), which is an abstract structure that generalizes the notion of linear independence to set systems.
We formalize this first problem as follows.
\begin{problem}\label{prob:problem1} \emph{Optimal matching of redundant robots under cardinality constraints:}
Given $N$ available robots, $M$ goals with uncertain robot-to-goal assignment costs, and an initial assignment $\mathcal{O}$, find a matching $\mathcal{A} \subseteq \mathcal{F}_{\mathcal{O}}$ of redundant robots to goals such that the average cost over all goals is minimized, and the total number of robots deployed is $N_{\mathrm{d}}$. This is formally stated as:
\begin{eqnarray}
    \label{eq:opt_prob1}
    \mathop{\mathrm{argmin}}\limits_{\mathcal{A} \subseteq \mathcal{F}_{\mathcal{O}}} && J_{\mathcal{O}}(\mathcal{A}) \label{eq:objective1}\\
    \mathrm{subject~to}    
     &&   \forall i |\{ j | (i,j) \in \mathcal{A} \cup \mathcal{O} \} | \leq 1 \label{eq:constraint1}\\
     &&    |\mathcal{A}| = N_{\mathrm{d}} - M \label{eq:constraint2}
     \end{eqnarray}
\end{problem}

Our second problem is closely related to the first. Instead of constraining the maximum deployment size and minimizing the cost, we now constrain the average aggregate cost and minimize the deployment size. The cost budget $\xi$ is pre-defined by the user. Such a constraint is useful when a certain performance guarantee is required, and when we wish to deploy the least possible redundant resources to meet this requirement.
We formalize this second problem as follows.
\begin{problem}\label{prob:problem2} \emph{Optimal matching of redundant robots under a cost budget:}
Given $N$ available robots, $M$ goals with uncertain robot-to-goal assignment costs, and an initial assignment $\mathcal{O}$, find a minimally sized matching $\mathcal{A} \subseteq \mathcal{F}_{\mathcal{O}}$ of redundant robots to goals, such that the average aggregate cost over all goals $J_{\mathcal{O}}(\mathcal{A})$ is smaller than a given cost budget $\xi \in \mathbb{R}$. This is formally stated as:
\begin{eqnarray}
    \label{eq:opt_prob2}
    \mathop{\mathrm{argmin}}\limits_{\mathcal{A} \subseteq \mathcal{F}_{\mathcal{O}}} && |\mathcal{A}| \label{eq:objective2}\\
    \mathrm{subject~to}    
     &&   \forall i |\{ j | (i,j) \in \mathcal{A} \cup \mathcal{O} \} | \leq 1 \\
     &&   J_{\mathcal{O}}(\mathcal{A}) \leq \xi \label{eq:J0constraint}
     \end{eqnarray}
\end{problem}


\section{Application to Transport Networks}
\label{sec:application}

We are interested in applications on transport networks, and use graphs to represent possible robot routes (from origins to goals). 
We represent routes via a weighted directed graph, $\mathcal{G}= (\mathcal{V},\mathcal{E}, \mathcal{W})$. Vertices in the set $\mathcal{V}$ represent geographic locations. Nodes $u$ and $v$ are connected by an edge if $(u, v) \in \mathcal{E}$. We assume the graph $\mathcal{G}$ is a strongly connected graph, i.e., a path exists between any pair of vertices. 
Paths between a same origin $i$ and goal $j$ are distinct if there is at least one edge in one path that is not present in the other path. 

The random variable $C_{ijk}$ captures the estimated travel time for robot $i$ to reach a goal $j$ via path $k$. The following two sections elaborate specific implementations of 
$\mathbb{E}[C_{ijk}]$ for the two causes of uncertainty that we consider in this work. Furthermore, they elucidate a \emph{first-come, first-to-serve} concept, by which only the fastest robot to reach a task actually services it. Redundancy, as defined in Def.~\ref{def:aggregate}, allows us to reduce the waiting time at the goal locations: when multiple robots travel to the same destination, only the travel time of the fastest robot counts. This is specified in the following definition.

\begin{definition}[Effective Waiting Time]
\label{def:waiting_time}
Since only the first robot's arrival defines the effective waiting time at a goal node $j$, the aggregate cost function $\Lambda$ (see Def.~\ref{def:aggregate}) is equivalent to the \emph{minimum} operator. The effective waiting time (cost) at goal $j$ is
\begin{equation}
\Lambda(I_j(\mathcal{A})) \triangleq \min\{ C_{ijk} | (i,j,k) \in I_j(\mathcal{A})\}. \label{eq:min}
\end{equation}
\end{definition}


\subsection{Application to Transport with Uncertain Travel Time}
\label{sec:application_edges}

We consider a robot $i$ at node $r_i$ and a goal $j$ at node $g_j$.
We define a random variable $w_{uv} \in \mathcal{W}$ that represents the time needed to traverse an edge $(u,v)$. The set of weights $\mathcal{W}$ can be sampled from a distribution $\mathcal{D}_w$. 
We formulate thee stimated travel time from node $r_i$ to goal $g_j$ on path $k$ as:
\begin{eqnarray}
\label{eq:cost_edges}
C_{ijk} &=& \sum_{(u,v) \in \mathcal{S}_{r_i,g_j,k}} w_{uv}, \mathrm{~~and~thus~~} \\
\mathbb{E}[C_{ijk}] &=& \sum_{(u,v) \in \mathcal{S}_{r_i,g_j,k}} \mathbb{E}[w_{uv}],
\end{eqnarray}
where $\mathcal{S}_{r_i,g_j,k}$ is the set of edges on path $k$ between nodes $r_i$ and $g_j$. 
The distribution $\mathcal{D}$, from which the $C_{ijk}$ are sampled, is defined through~\eqref{eq:cost_edges} and the distribution $\mathcal{D}_w$.

\emph{First-come, first-to-serve:} The following example illustrates this principle when travel time on transport edges is uncertain.
We consider a simple redundant assignment problem, as illustrated in Fig.~\ref{fig:graph_sketch}, with multiple path choices. In this particular example, robot $\mathcal{R}_1$ has already been assigned to the goal and $\mathcal{R}_2$ must choose between paths B and C. Path B is correlated with path A (they share two edges). Path C appears to take longer, and the intuitive choice would be path B. However, as Fig.~\ref{fig:corr_gaussians} shows, there is a small chance that path C leads to an improved waiting time at the goal. In other words, if robot $\mathcal{R}_2$ selects path C, there is a small chance that it is the first robot to arrive at the goal, improving upon the performance of robot $\mathcal{R}_1$ (i.e., the arrival of robot $\mathcal{R}_2$ at the goal defines the effective waiting time, and it is the \emph{first-to-serve}). 
%
%
\begin{figure}[tb]
\centering
\psfrag{b}[cc][][0.7]{path B}
\psfrag{a}[cc][][0.7]{path A}
\psfrag{c}[cc][][0.7]{path C}
\psfrag{v}[cc][][0.8]{$\mathcal{R}_1$}
\psfrag{x}[cc][][0.8]{$\mathcal{R}_2$}
\psfrag{g}[cc][][0.8]{Goal}
\psfrag{i}[cc][][0.7][90]{PDF}
\psfrag{e}[cc][][0.7]{Travel Time [s]}
\subfigure{\includegraphics[width=0.65\columnwidth]{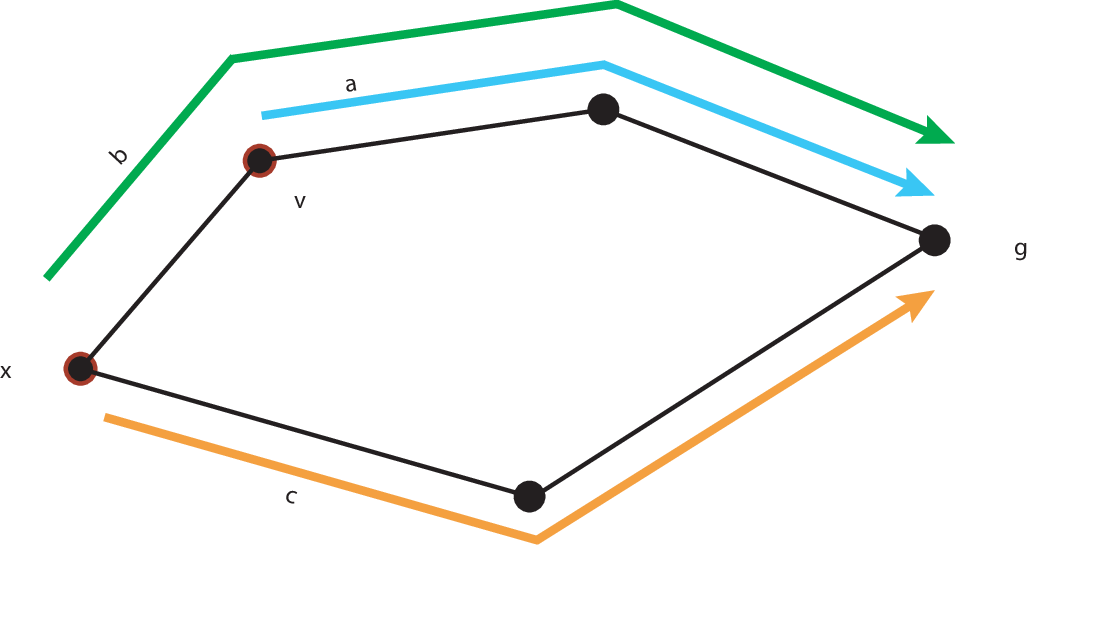}}\\
\subfigure{\includegraphics[width=0.65\columnwidth]{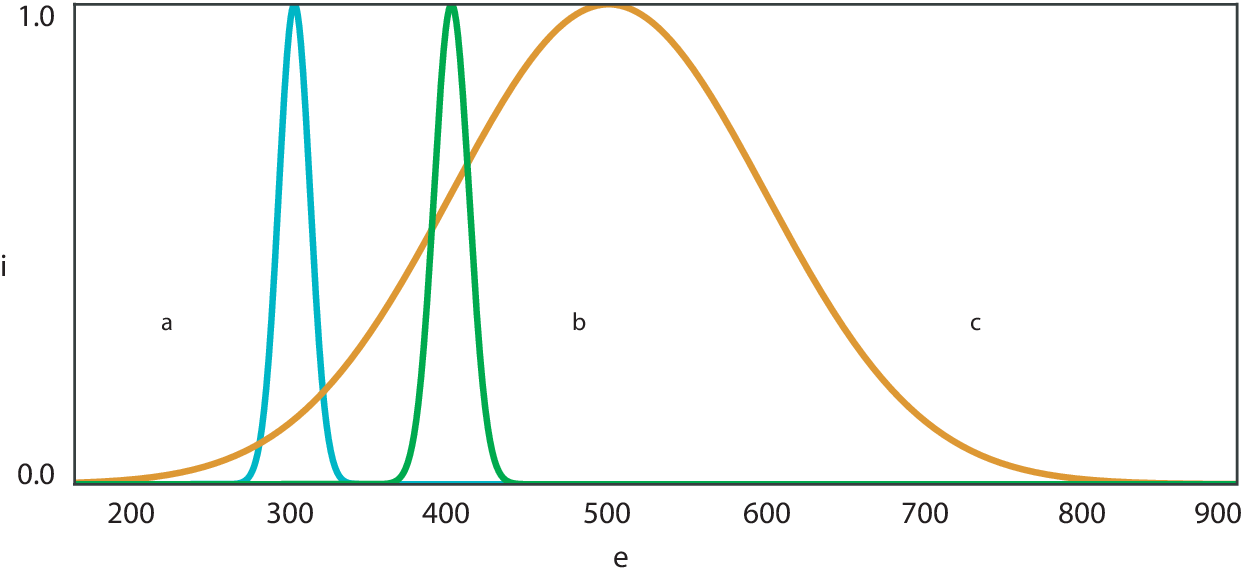}}
\caption{Example scenario. Robot $\mathcal{R}_1$ is assigned to the goal via path A. Robot $\mathcal{R}_2$ is a redundant robot, and must choose between paths B and C.
\label{fig:graph_sketch}}
\end{figure}
\begin{figure}[tb]
\centering
\psfrag{a}[cc][][0.8][90]{Travel time along path \bf{B}}
\psfrag{x}[cc][][0.8]{Travel time along path \bf{A}}
\psfrag{c}[cc][][0.8][90]{Travel time along path \bf{C}}
\subfigure[]{\includegraphics[width=0.45\columnwidth]{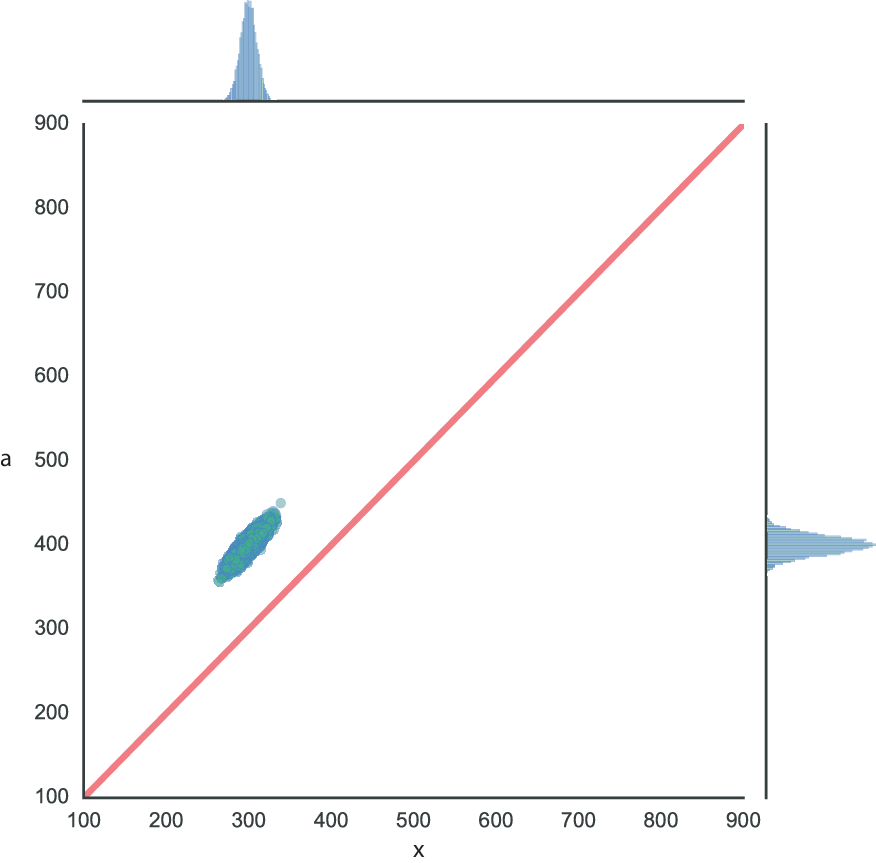}}\hspace{0.3cm}
\subfigure[]{\includegraphics[width=0.45\columnwidth]{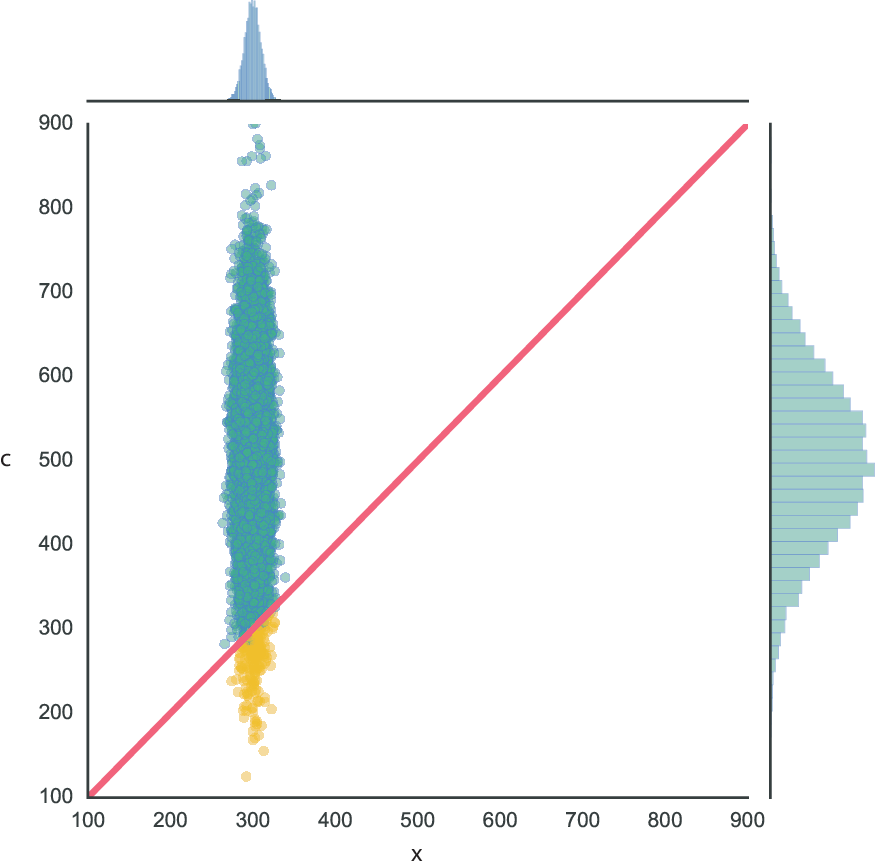}}
\caption{Joint distributions of travel times, for (a) paths A and B, and (b) paths A and C. The Pearson correlation coefficient is (a) 0.89 and (b) 0.0. The red line shows axis equality (i.e., equal travel times). If robot $\mathcal{R}_2$ chooses path C, there is a small chance (see datapoints in yellow) that the waiting time at the goal will be improved upon, despite being slower than path B on average.
\label{fig:corr_gaussians}}
\end{figure}
%


\subsection{Application to Transport with Uncertain Robot Positions}
\label{sec:application_nodes}

We consider that robot $i$ estimates its origin node to be $\tilde{r}_i$. 
The true origins of all robots can hence be defined as random variables $\{{r}_i | i = 1, \ldots, N \}$, sampled from a distribution $\mathcal{D}_v$.
Similar to the above, we formulate the travel time of robot $i$ to goal $j$ on path $k$ as
\begin{eqnarray}
\label{eq:cost_nodes}
C_{ijk} &=& \sum_{(u,v) \in \mathcal{S}_{r_i,g_j,k}} w_{uv},
\end{eqnarray}
where $\mathcal{S}_{r_i,g_j,k}$ is the set of edges on path $k$ between the robot $i$'s possible origin ${r}_i$, and goal $g_j$. 
When $w_{uv}$ are fixed, the only source of uncertainty comes from noisy position estimates of robots at their origin nodes. In this case, the average travel time for robot $i$ is
\begin{eqnarray}
\mathbb{E}[C_{ijk}] &=& \sum_{{r}_i \in \mathcal{V}} P({r}_i | \tilde{r}_i) \sum_{(u,v) \in \mathcal{S}_{{r}_i,g_j,k}} w_{uv},
\end{eqnarray}
where $P(r_i | \tilde{r}_i)$ is the marginal probability that the true origin of robot $i$ is ${r}_i$ when its estimated position is $\tilde{r}_i$. We note that when there is no noise on the transport network edges (i.e., $w_{uv}$ are fixed), we will assume that each robot takes the shortest path to its assigned goal, and hence, $K=1$. 
The distribution $\mathcal{D}$, from which the $C_{ijk}$ are sampled, is defined through~\eqref{eq:cost_nodes} and the distribution $\mathcal{D}_v$.

\begin{figure}[tb]
\centering
\psfrag{i}[cc][][0.9]{Goal}
\psfrag{a}[ct][][0.9]{$\tilde{r}_1$}
\psfrag{u}[ct][][0.9]{$\tilde{r}_2$}
\psfrag{o}[ct][][0.9]{${{r}_1}$}
\psfrag{e}[ct][][0.9]{${{r}_2}$}
\psfrag{c}[ct][][0.9]{${g_1}$}
\psfrag{1}[cc][][0.9]{$\mathcal{D}_v$}
\psfrag{3}[cc][][0.7]{$\mathcal{R}_1$}
\psfrag{4}[cc][][0.7]{$\mathcal{R}_2$}
\includegraphics[width=0.95\columnwidth]{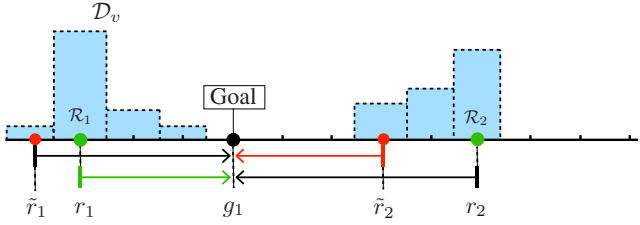}
\caption{One-dimensional sketch of a redundant assignment of two robots to one goal. Based on the noisy positions (in red), robot $\mathcal{R}_2$ is closer to the goal. However, based on the true positions (in green), robot $\mathcal{R}_1$ is closer to the goal, and is effectively the one to service the task at the goal.
\label{fig:application}}
\end{figure}

\emph{First-come, first-to-serve:} Fig.~\ref{fig:application} illustrates this principle when robot positions are uncertain. Based on the noisy position estimates, robot $\mathcal{R}_2$ is closer to the goal, and is assigned to it. A redundant assignment of both robots to the goal reduces the effective waiting time at the goal, with robot $\mathcal{R}_1$ being the first robot to arrive, since, based on the true positions, robot $\mathcal{R}_1$ is closer to the goal than robot $\mathcal{R}_2$.






\section{Method} 
\label{sec:method}
Our method is underpinned by the following key insight: as we assign additional robots to a given goal, the waiting time at that goal decreases by durations of diminishing length.
This property is known as \emph{supermodularity}.
In the following, we begin by introducing the fundamental concepts, and demonstrate the supermodular structure of our problem. Building on this, in Section~\ref{sec:greedy}, we elaborate a polynomial-time algorithm that addresses Problem~\ref{prob:problem1} and Problem~\ref{prob:problem2}.


\subsection{Preliminaries}
\label{sec:background}

A supermodular function is a set function. Given a finite set $\mathcal{F}$, it is defined as $J: 2^{\mathcal{F}} \mapsto \mathbb{R}$, assigning a scalar to any subset of $\mathcal{F}$. 

\begin{definition}[Marginal decrease]
For a finite set $\mathcal{F}$ and a given set function $J: 2^{\mathcal{F}} \mapsto \mathbb{R}$, the marginal decrease of $J$ at a subset $\mathcal{A} \subseteq \mathcal{F}$ with respect to an element $x \in \mathcal{F} \setminus \mathcal{A}$ is:
\begin{eqnarray}
\Delta_{J}(x|\mathcal{A}) \triangleq J(\mathcal{A}) - J(\mathcal{A} \cup \{x\}).
\end{eqnarray}
\end{definition}

\begin{definition}[Supermodular]\label{def:supermodular}
Let $J: 2^{\mathcal{F}} \mapsto \mathbb{R}$ and $\mathcal{A} \subseteq \mathcal{B} \subseteq \mathcal{F}$. The set function $J$ is supermodular if and only if for any $x \in \mathcal{F} \setminus \mathcal{B}$:
\begin{eqnarray}
\Delta_{J}(x| \mathcal{A}) \geq \Delta_{J}(x| \mathcal{B})
\end{eqnarray}
\end{definition}
The definition implies that adding an element $x$ to a set $\mathcal{A}$ results in a larger marginal decrease than when $x$ is added to a superset of $\mathcal{A}$. This property is known as a property of diminishing returns from an added element $x$ as the set it is added to grows larger. 
\begin{remark}\label{remark:submodular}
A function $J$ is submodular if $-J$ is supermodular. 
\end{remark}

\begin{definition}[Monotone non-increasing set function]\label{def:monotone}
A set function $J: 2^{\mathcal{F}} \mapsto \mathbb{R}$ is monotone non-increasing if for any $\mathcal{A} \subseteq \mathcal{B}$, $J(\mathcal{A}) \geq J(\mathcal{B})$.
\end{definition}

\begin{definition}[Matroid] \label{def:matroid}
Given a finite ground set $\mathcal{F}$ and $\mathcal{I} \subseteq 2^{\mathcal{F}}$ a family of subsets of $\mathcal{F}$, an independence system is an ordered pair $(\mathcal{F}, \mathcal{I})$ with the following two properties: {(i)} $\emptyset \in \mathcal{I}$, and {(ii)} for every $\mathcal{A} \in \mathcal{I}$, with $\mathcal{B} \subseteq \mathcal{A}$, implies that $\mathcal{B} \in \mathcal{I}$.
The second property is known as downwards-closed --- in other words, every subset of an independent set is independent. 
An independence system $(\mathcal{F}, \mathcal{I})$ is a matroid if it also satisfies the augmentation property, that is, for every $\mathcal{A}, \mathcal{B} \in \mathcal{I}$, with $|\mathcal{A}| > |\mathcal{B}|$, there exists an element $a \in \mathcal{A} \setminus \mathcal{B}$ such that $\{a\} \cup \mathcal{B} \in \mathcal{I}$.
\end{definition}

For an optimization problem $\min J(\mathcal{A})$ such that $\mathcal{A} \subset \mathcal{I}$, where $J$ is supermodular and $\mathcal{I}$ is an independence system, we can apply a \emph{greedy} approximation algorithm.
This approach works as follows.
At each iteration $k$, an element $a$ is added to the solution set $\mathcal{A}_{k-1}$ such that it maximizes the marginal decrease given this current set,
\begin{equation}
\label{eq:greedy_max}
a \gets \mathrm{argmax}_{ x \in \mathcal{F}_E (\mathcal{A}_{k-1}, \mathcal{I}) } \Delta_J (x|\mathcal{A}_{k-1})
\end{equation}
where $\mathcal{F}_{E}$ denotes the set of eligible elements, given the independence system $\mathcal{I}$ and ground set $\mathcal{F}$, defined as
\begin{equation}
\label{eq:greedy_set}
\mathcal{F}_{E}(\mathcal{A}_{k}, \mathcal{I}) \triangleq \{ x \in \mathcal{F} \setminus \mathcal{A}_k \,|\, \mathcal{A}_k \cup \{ x\} \in \mathcal{I} \}.
\end{equation} 
A key property is that optimization with this algorithm yields a $1/2$ approximation ratio~\cite{fisher:1978}.
In our case, this is equivalent to \emph{Greedy} returning a set $\mathcal{A}_G$ with ratio $J(\mathcal{A}_G) \leq \frac{1}{2} \, (J^\star + J_0)$ where $J^\star$ is the optimal cost, and $J_0$ is the system's baseline performance (without redundant assignments).
For matroid rank $r$ and ground set size $n$, \emph{Greedy} requires $O(nr)$ calls to the objective function.

\subsection{Supermodularity of cost functional}
\label{sec:theorems}

The following derivations prove the supermodularity of our cost function.
\begin{lemma}\label{lemma:sum}(From~\cite{fujishige:2005}, Section 3.1 (c))
Any nonnegative finite weighted sum of supermodular functions is supermodular.
\end{lemma}

\begin{lemma}\label{lemma:min} Given two sets $\mathcal{A}$ and $\mathcal{B}$ that are composed of random variables, we have
$\mathbb{E}[\min \mathcal{A}] \geq \mathbb{E}[\min (\mathcal{A} \cup \mathcal{B})]$.
\end{lemma}
\begin{proof}
Consider $X$ and $Y$ two random variables. Since $\min\{X\} \geq \min\{X,Y\}$, 
we have that $\mathbb{E}[ \min  \{X\} ] \geq \mathbb{E}[ \min \{ X,Y \}]$, which follows directly from the monotonicity of the expectation operator.
We can now see that $\mathbb{E}[\min \mathcal{A}] = \mathbb{E}[\min \{\min \mathcal{A}\}] \geq \mathbb{E}[\min \{\min \mathcal{A}, \min \mathcal{B} \}] = \mathbb{E}[\min (\mathcal{A} \cup \mathcal{B})]$, since both $\min \mathcal{A}$ and $\min \mathcal{B}$ are random variables.
\end{proof}

Based on these definitions, we show that our cost function $J_{\mathcal{O}}(\mathcal{A})$ in~\eqref{eq:objective1} and~\eqref{eq:J0constraint} has the following property:
\begin{theorem}\label{theorem:supermodular}
The cost function $J_{\mathcal{O}}(\mathcal{A})$ is a non-increasing supermodular function of the set $\mathcal{F}_\mathcal{O}$.
\end{theorem}

\begin{proof}
Based on Lemma \ref{lemma:sum}, it suffices to show that $\mathbb{E}[\Lambda(I_j(\mathcal{A}\cup\mathcal{O}))] = \mathbb{E}[\min\{C_{ijk} | (i,j,k) \in \mathcal{A} \cup \mathcal{O} \}]$ is supermodular for all $j$. 
This is equivalent to showing that $\mathbb{E}[\min \mathcal{A}]$ is a supermodular monotone non-increasing function of $\mathcal{A}$, where $\mathcal{A}$ denotes a set of random variables. 

By Def. \ref{def:supermodular}, this is true if and only if 
{
\begin{eqnarray}\label{eq:super1}
&& \mathbb{E}[\min \mathcal{A} ]  - \mathbb{E}[\min (\mathcal{A} \cup \{x\})] \nonumber \\ 
&& \geq \mathbb{E}[\min \mathcal{B} ]  - \mathbb{E}[\min (\mathcal{B} \cup \{x\})]
\end{eqnarray}}
with $\mathcal{A} \subseteq \mathcal{B} \subseteq \mathcal{F}_{\mathcal{O}}$ and $x \in \mathcal{F}_{\mathcal{O}} \setminus \mathcal{B}$, where $\mathcal{F}_{\mathcal{O}}$ denotes the ground set. Consider $\mathcal{B} = \mathcal{A} \cup \mathcal{Y}$ where $\mathcal{Y} \subset \mathcal{F}_{\mathcal{O}}$. Given the linearity of the expectation operator, we rewrite~\eqref{eq:super1} as
{
\begin{eqnarray}\label{eq:super2}
&& \mathbb{E}[\min \mathcal{A} - \min (\mathcal{A} \cup \{x\})] \nonumber \\
&& \geq \mathbb{E}[\min (\mathcal{A} \cup \mathcal{Y}) - \min (\mathcal{A} \cup \mathcal{Y} \cup \{x\})]
\end{eqnarray}
}
By Lemma~\ref{lemma:min} we know that $\min \mathcal{A} \geq \min (\mathcal{A} \cup \mathcal{Y} \cup \{x\})$. Hence, it remains to show that the difference on the left side of the inequality~\eqref{eq:super2} is greater or equal to the difference on its right side. To do this, we consider the following two possible cases.

\emph{Case 1:} $\min (\mathcal{A} \cup \{x\}) \geq \min (\mathcal{A} \cup \mathcal{Y})$. This implies that $\min (\mathcal{A} \cup \mathcal{Y}) - \min (\mathcal{A} \cup \mathcal{Y} \cup \{x\}) = 0$. By Lemma~\ref{lemma:min}, we have that $\min (\mathcal{A} \cup \{x\}) \leq \min \mathcal{A}$.  Thus, $\min \mathcal{A} - \min(\mathcal{A} \cup \{x\}) \geq 0$, and given the monotonicity of the expectation operator, the inequality in \eqref{eq:super2} holds.

\emph{Case 2:} $\min (\mathcal{A} \cup \{x\}) < \min (\mathcal{A} \cup \mathcal{Y})$. This implies that $\min(\mathcal{A} \cup \mathcal{Y} \cup \{x\}) = \min (\mathcal{A} \cup \{x\})$. By Lemma~\ref{lemma:min}, we have that $\min (\mathcal{A} \cup \mathcal{Y}) \leq \min \mathcal{A}$. Thus, we have that $\min \mathcal{A} - \min (\mathcal{A} \cup \{x\}) \geq \min (\mathcal{A} \cup \mathcal{Y}) - \min (\mathcal{A} \cup \{x\})$, and given the monotonicity of the expectation operator, the inequality in \eqref{eq:super2} holds.

To show that ${J}_\mathcal{O}$ is a monotone non-increasing set function (as defined by Def.~\ref{def:monotone}), it suffices to substitute $\mathbb{E}[\min \mathcal{A}]$ into the inequality of Lemma~\ref{lemma:min}.
\end{proof}
The results above establish our problem of selecting redundant robot assignments to minimize the effective waiting time at destinations as a problem of supermodular minimization. 

\section{Algorithmic Approach} 
\label{sec:greedy}

If we are given a supermodular objective function that satisfies a matroid constraint, we can employ a greedy algorithm to solve our problem within known optimality bounds. However, to maintain the efficiency of such a greedy assignment algorithm, we need to ensure that the evaluation of the objective function itself is efficient (and can be computed in polynomial time). Towards this end, we develop a dynamic programming (DP) approach that hinges on a definition of \emph{incrementally computable functions}, which we apply to our redundant assignment problem. The following paragraphs elaborate our methodology --- first, in Sec.~\ref{sec:matroid}, we show that the matroid constraint applies to our problem setting, and second, in Sec.~\ref{sec:dp_greedy}, we show how Greedy is implemented efficiently through a dynamic programming approach. The resulting routine is shown in Algorithm~\ref{alg1}.

\subsection{Matroid Constraint}
\label{sec:matroid}
In the following we show that the problem of assigning redundant robots with multiple path options satisfies the properties of a matroid. Following constraints~\eqref{eq:constraint1} and~\eqref{eq:constraint2}, our problem considers the matroid $(\mathcal{F}_\mathcal{O}, \mathcal{I}_\mathcal{O})$, with
\begin{eqnarray}
\mathcal{I}_\mathcal{O} \triangleq && \{ \mathcal{A} | \mathcal{A} \subseteq \mathcal{F}_\mathcal{O} \, 
 \land\,  |\mathcal{A}| \leq N_{\mathrm{d}}-M  \, \nonumber \\
 && \land\, \forall i |\{j|(i,j,k)\in \mathcal{A} \cup \mathcal{O}\}| \leq1  \}. \label{eq:IO}
\end{eqnarray}
By the definition of a matroid, any valid assignment must be an element of the family of independent sets $\mathcal{I}_\mathcal{O}$.
Firstly, the empty set is a valid solution, in which case our objective function is reduced to $J_0$, as given by~\eqref{eq:J0}. Secondly, our system is downwards-closed: for any valid robot-to-goal assignment $\mathcal{A} \in \mathcal{I}_\mathcal{O}$, any subset of assignments $\mathcal{B} \subseteq \mathcal{A}$ is also a valid assignment by~\eqref{eq:constraint1}. Thirdly, we can show that our system satisfies the augmentation property. For any two valid assignments $\mathcal{A}$ and $\mathcal{B}$, $|\mathcal{A}| > |\mathcal{B}|$ implies that there is at least one robot assigned to a goal in set $\mathcal{A}$ that is unassigned in set $\mathcal{B}$, irrespective of what path was selected. Hence, adding that robot-to-goal assignment to set $\mathcal{B}$ still satisfies~\eqref{eq:constraint1} and maintains the validity of the solution. We note that the augmentation property implies that all maximal solution sets have the same cardinality $N_{\mathrm{d}} - M$, which corresponds to the \emph{rank} of our matroid. 

\subsection{Greedy Assignment with Dynamic Programming}
\label{sec:dp_greedy}
Our work considers uncertainty models, represented by arbitrary distributions that are also capable of capturing correlations between random variables. Our approach is to consider a sampling-based method that takes $S$ samples from the $MNK$-dimensional joint distribution~\footnote{We note that if an analytical model is known, this can be used instead.} $\mathcal{D}$. 
Our aim is to ensure that the computation of the aggregate cost $\Lambda$ that assembles the performance of the robot coalition (see Def.~\ref{def:aggregate}) does not incur additional complexity that depends on the number of robots $N$, the deployment size $N_{\mathrm{d}}$, or number of tasks $M$. 
Our insight is that, in a number of practical cases, $\Lambda$ is a \emph{distributive aggregate function} and is {incrementally computable}~\cite{palpanas:2002}.
This allows us to implement a dynamic programming approach, as shown in Algorithm~\ref{alg1}.

\begin{definition}[Distributive Aggregate Function]
We define a class $\Delta$ of \uline{distributive} aggregate functions $\delta: A \mapsto \mathbb{R}$, for $A \subset \mathbb{R}$, such that $\delta \in \Delta$ only if $\delta(A \cup x)$ can be computed incrementally, as a function of the old value $\delta(A)$ and new value $x$ only.
\end{definition}

\begin{proposition}\label{theorem:alg_greedy}
Algorithm~\ref{alg1} is a valid instantiation of Greedy, and has complexity $O((N_{\mathrm{d}}-M) N M K S)$, if \emph{(i)} $\Lambda$ is a distributive aggregate function, \emph{(ii)} $J_{\mathcal{O}}$ is supermodular, and \emph{(iii)} $(F_{\mathcal{O}}, \mathcal{I}_\mathcal{O})$ is a matroid constraint.
\end{proposition}
%

\begin{algorithm}[tb]
\caption{{\small Greedy Redundant Assignment with DP}}
\label{alg1}
\begin{algorithmic}[1]
{\small
\REQUIRE Graph $\mathcal{B} = (\mathcal{U},\mathcal{F}, \mathcal{C})$, initial assignment $\mathcal{O}$ \newline
 Problem~\ref{prob:problem1}: size of deployment $N_{\mathrm{d}}$,  cost budget $\xi = 0$ \newline
 Problem~\ref{prob:problem2}: size of deployment $N_{\mathrm{d}} = N$,  cost budget $\xi$ 
\ENSURE Set of edges $\mathcal{A}$ defining redundant assignments
\STATE $\mathcal{A}_G \gets \emptyset$
\STATE $\mathcal{F}_{\mathcal{O}} \gets \mathcal{F} \setminus \mathcal{O} $
\STATE $\mathcal{I}_{\mathcal{O}} \gets $ Eq.~\eqref{eq:IO}
\STATE $\hat{\mathcal{C}} \gets \mathrm{sample}~S~\mathrm{samples~from~} MNK\mathrm{-dim.~distrib.}~\mathcal{D}$
\FOR {$\hat{C}_{ijk} \in \hat{\mathcal{C}}$}
\STATE $\texttt{samples}[(i,j,k)] \gets \hat{C}_{ijk}$
\ENDFOR
\FOR {$(i,j,k) \in \mathcal{O}$}
\STATE $\texttt{state}[j] \gets \texttt{samples}[(i,j,k)]$
\ENDFOR \label{line:end_init}
\FOR{$d \in \{1,\ldots,N_{\mathrm{d}}-M\}$}
\IF {{$J_{\mathcal{O}}(\mathcal{A}_G) \leq \xi$}} \label{line:p2_1}
\STATE {\texttt{break}} 
\ENDIF \label{line:p2_2}
\STATE $\Delta_{J_{\mathcal{O}}}^{\star} \gets -\infty$
\STATE  $\mathcal{F}_{\mathcal{O},E} \gets \{(i,j,k) \in \mathcal{F}_{\mathcal{O}} \setminus \mathcal{A}_G \,| \, \mathcal{A}_G \cup \{(i,j,k)\} \in \mathcal{I}_{\mathcal{O}} \}$ \label{line:eligible}
\FOR{ $(i,j,k) \in \mathcal{F}_{\mathcal{O},E}$}
\STATE $\texttt{curr} \gets \frac{1}{S} \sum_{z=1}^S \texttt{state}[j]_z$
\STATE $\texttt{new} \gets \frac{1}{S} \sum_{z=1}^S \Lambda(\texttt{state}[j]_z, \texttt{samples}[(i,j,k)]_z)$ \label{line:new_state}
\STATE $\Delta_{J_{\mathcal{O}}} \gets \texttt{curr} - \texttt{new}$
\IF {$\Delta_{J_{\mathcal{O}}} > \Delta_{J_{\mathcal{O}}}^\star$}
\STATE {$\Delta_{J_{\mathcal{O}}}^\star \gets \Delta_{J_{\mathcal{O}}}$} \label{line:best}
\STATE $(i^\star,j^\star,k^\star) \gets (i,j,k)$
\ENDIF
\ENDFOR
\STATE $\mathcal{A}_G \gets \mathcal{A}_G \cup (i^\star,j^\star,k^\star)$
\STATE $\texttt{state}[j] \gets \Lambda.(\texttt{state}[j], \texttt{samples}[(i^\star,j^\star,k^\star)])$ \label{line:element}
\ENDFOR
\RETURN $\mathcal{A}_G$
}
\end{algorithmic}
\end{algorithm}

Algorithm~\ref{alg1} works as follows. First, the input designates to options: for Problem~\ref{prob:problem1}, we set a maximum deployment size and do not consider a cost budget; for Problem~\ref{prob:problem2}, the maximum deployment size is limited only by the total number of robots $N$, and the cost budget $\xi$ is set to some feasible value (e.g., within a 1/2 approximation ratio of the optimal performance).
Lines 1-\ref{line:end_init} initialize the data structures. 
In particular, we pre-sample a fixed set of $S$ samples (which amounts to sampling $S$ values for each of the $MNK$ matchings). Pre-sampling allows the algorithm to maintain the supermodularity property.
For the remaining number of robots to be deployed, we proceed with a greedy assignment. We note that lines~\ref{line:p2_1}-~\ref{line:p2_2} relate to Problem~\ref{prob:problem2}, and can be omitted if only Problem~\ref{prob:problem1} is to be solved. Similarly, for Problem~\ref{prob:problem2}, $N_{\mathrm{d}}$ is set to $N$.
Line~\ref{line:eligible} constructs the set of eligible assignments, as in Eq.~\eqref{eq:greedy_set}. Then, for all eligible assignment candidates, we compute the marginal cost decrease incurred by adding that assignment to goal $j$. In order to do this, Line~\ref{line:new_state} computes the new aggregate cost function. This is done incrementally, since $\Lambda$ is a distributive aggregate function. Overall, this inner for-loop is equivalent to Eq.~\eqref{eq:greedy_max}, which allows line~\ref{line:best} to retain the best assignment candidate. We then add the best candidate to the current solution, and  update the aggregate cost incurred at the goal the new robot was assigned to. Line~\ref{line:element} uses an element-wise operator.
Our approach requires $O((N_{\mathrm{d}}-M) N M K)$ calls to the objective function, and the objective function is computed in $O(S)$.

The transport network application described in Sec.~\ref{sec:application} satisfies the conditions in Proposition~\ref{theorem:alg_greedy}: Since~\eqref{eq:min} is supermodular, and the \emph{minimum} operator in Def.~\ref{def:waiting_time} is a distributive aggregate function. It follows that the objective of minimizing the average effective waiting time is supermodular. The matroid constraint is trivially satisfied.

\begin{proposition}\label{theorem:bound_p1}
Consider $J_0$, given by~\eqref{eq:J0}, as the maximum possible cost of the system, when no redundant robots are deployed, and each goal is assigned exactly one robot. Let $J^\star_{\mathcal{O}}$ be the optimal value of~\eqref{eq:objective1}. Then, Algorithm \ref{alg1} returns a set $\mathcal{A}_G$ satisfying
\begin{eqnarray}\label{eq:bound_vanilla}
J_{\mathcal{O}}(\mathcal{A}_G) \leq \frac{1}{2} \, (J_{\mathcal{O}}^\star + J_0).
\end{eqnarray}
The randomized continuous greedy algorithm (as in~\cite{badanidiyuru:2014,filmus:2014}) satisfies
\begin{eqnarray}\label{eq:bound_randomized}
J_{\mathcal{O}}(\mathcal{A}_G) \leq \left(1-\frac{1}{e} - \epsilon \right) \, J_{\mathcal{O}}^\star + \left(\frac{1}{e} + \epsilon \right) \, J_0.
\end{eqnarray}
\end{proposition}

\begin{proof}
Consider the function $Q(\mathcal{A}) = J_0 - J_{\mathcal{O}}(\mathcal{A})$.
Since $J_{\mathcal{O}}(\mathcal{A})$ is non-increasing and supermodular by Theorem~\ref{theorem:supermodular}, then, by Remark~\ref{remark:submodular}, $Q(\mathcal{A})$ is a monotone non-decreasing normalized submodular function. Hence, minimizing $J_{\mathcal{O}}(\mathcal{A})$ is equivalent to maximizing $Q(\mathcal{A})$. 

By Theorem 1.1 of~\cite{fisher:1978}, for a monotone nondecreasing submodular function $Q(\mathcal{A})$ subject to a matroid constraint, the greedy algorithm returns a set $\mathcal{A}_G$ satisfying $Q(\mathcal{A}_G) \geq \, Q^\star/2$, with $Q^\star \triangleq \max\{ Q(\mathcal{A}) | \mathcal{A} \in \mathcal{I}\}$ and where $M=(\mathcal{F}, \mathcal{I})$ is a matroid given by a membership oracle. 
Consequently, the solution $\mathcal{A}_G$ returned by Algorithm \ref{alg1} satisfies $Q(\mathcal{A}_G) \geq \, Q^\star/2$. We substitute the definition of $Q(\mathcal{A})$ into this equation to yield the result in~\eqref{eq:bound_vanilla}.

By Theorem 1.3 of~\cite{badanidiyuru:2014}, for a monotone nondecreasing submodular function $Q(\mathcal{A})$ subject to a matroid constraint, the randomized continuous greedy algorithm returns a set $\mathcal{A}_G$ satisfying $Q(\mathcal{A}_G) \geq (1- 1/e + \epsilon) \, Q^\star$. By the same procedure as above, substituting the definition of $Q(\mathcal{A})$ into this equation yields the result in~\eqref{eq:bound_randomized}.
\end{proof}

For completeness, we also pose a bound for the set size. The following is known to be true~\cite{feige:1998}:
\begin{proposition}[from~\cite{feige:1998,tzoumas:2016}]\label{theorem:bound_p2}
Consider $\mathcal{A}^\star$ as the optimal solution to Problem~\ref{prob:problem2}, and consider $\mathcal{A}_0, \mathcal{A}_1, \ldots$ the sequence of sets selected by Algorithm~\ref{alg1}. Let $k$ be the smallest index such that $J_{\mathcal{O}}(\mathcal{A}_k) \leq \xi$. Then, 
\begin{eqnarray}\label{eq:bound_p2}
\frac{|\mathcal{A}_k|}{|\mathcal{A}^\star|} \leq 1 + \log \frac{J_{\mathcal{O}}(\mathcal{O}) - J_{\mathcal{O}}(\mathcal{\emptyset})} {J_\mathcal{O}(\mathcal{O}) - J_{\mathcal{O}}(\mathcal{\mathcal{A}}_{k-1})}.  
\end{eqnarray}
\end{proposition}
%

\section{Evaluation} 
\label{sec:evaluation}

We evaluate our algorithm through a series of simulations. Section~\ref{sec:benchmarks} introduces the benchmark algorithms against which we compare our method. The following section, Sec.~\ref{sec:optimal}, compares our method to an exhaustive, optimal search strategy. The final two sections, Sec.~\ref{sec:results_node} and~\ref{sec:results_edge}, discuss the performance of our method for node position uncertainty and edge cost uncertainty, respectively. 

As shown in Algorithm~\ref{alg1}, the solutions to Problem~\ref{prob:problem1} and~\ref{prob:problem2} are coupled. In particular, it is possible to find an $N_{\mathrm{d}}$ such that the solution to Problem~\ref{prob:problem1} is the same as to Problem~\ref{prob:problem2}, and vice-versa. Hence, the following results focus mainly on Problem~\ref{prob:problem1}; we include results for Problem~\ref{prob:problem2} at the end of this section.

\begin{figure}
\centering
\psfrag{x}[lc][][0.8]{$N_{\mathrm{d}}$}
\psfrag{z}[lc][][0.8]{$\sigma$}
\psfrag{y}[cc][][0.7][90]{Normalized Waiting Time}
\psfrag{o}[lc][][0.6]{Hungarian}
\psfrag{n}[lc][][0.6]{Bound}
\psfrag{a}[lc][][0.6]{Random}
\psfrag{c}[lc][][0.6]{Best a-post.}
\psfrag{e}[lc][][0.6]{Greedy}
\psfrag{i}[lc][][0.6]{Optimal}
\psfrag{s}[lc][][0.6]{Slice-Greedy}
\psfrag{A}[cc][][0.8][90]{}
\subfigure[]{\includegraphics[height=3.6cm]{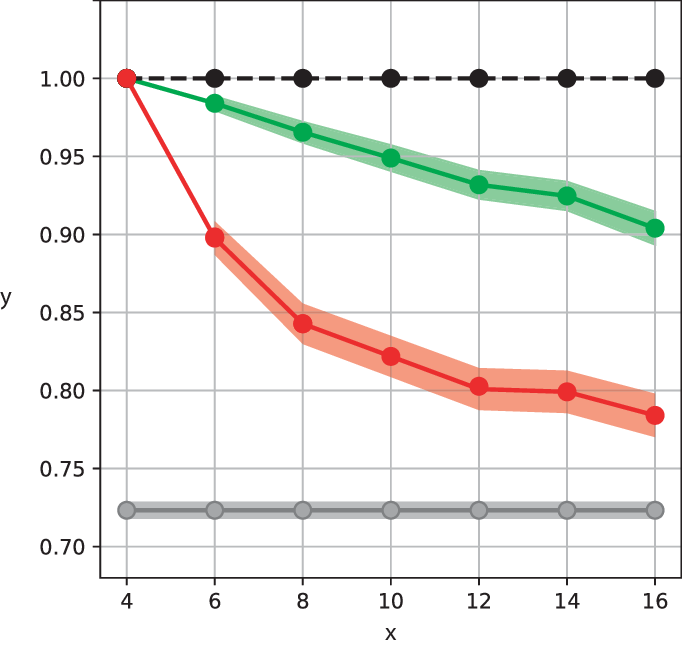}\label{fig:results_node_a}}
\subfigure[]{\includegraphics[height=3.6cm]{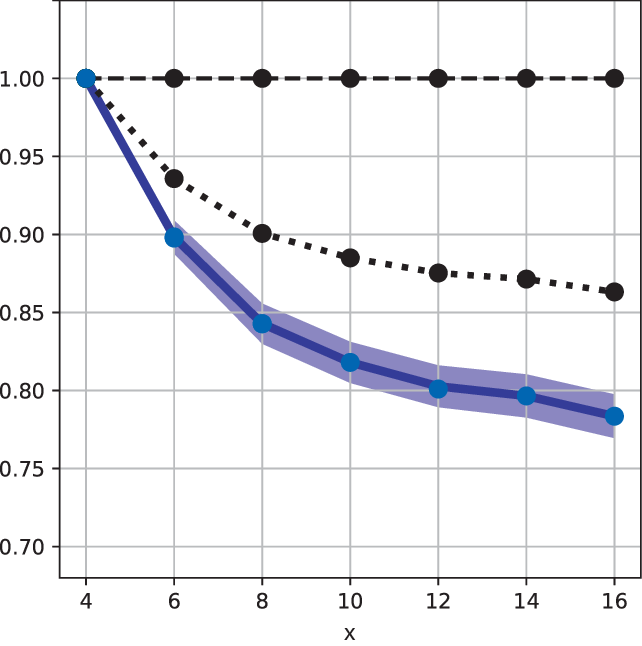}\label{fig:results_node_b}}\\
\subfigure[]{\includegraphics[height=3.6cm]{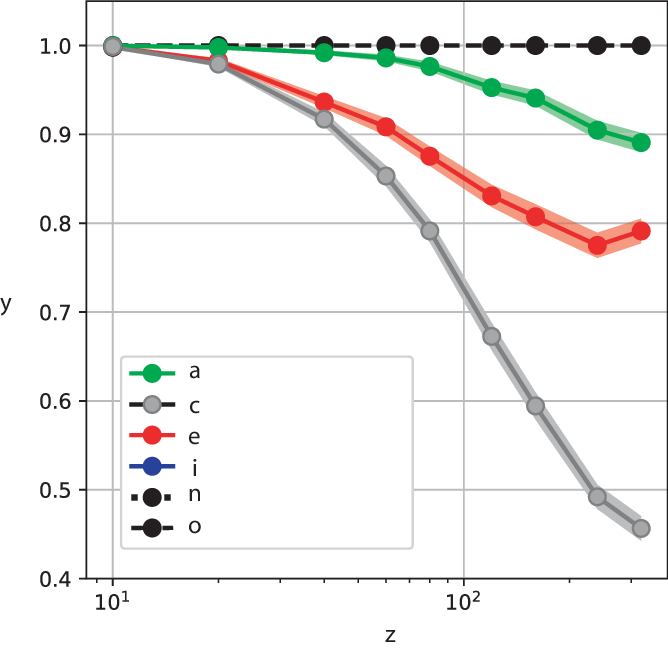}\label{fig:results_node_c}}
\subfigure[]{\includegraphics[height=3.6cm]{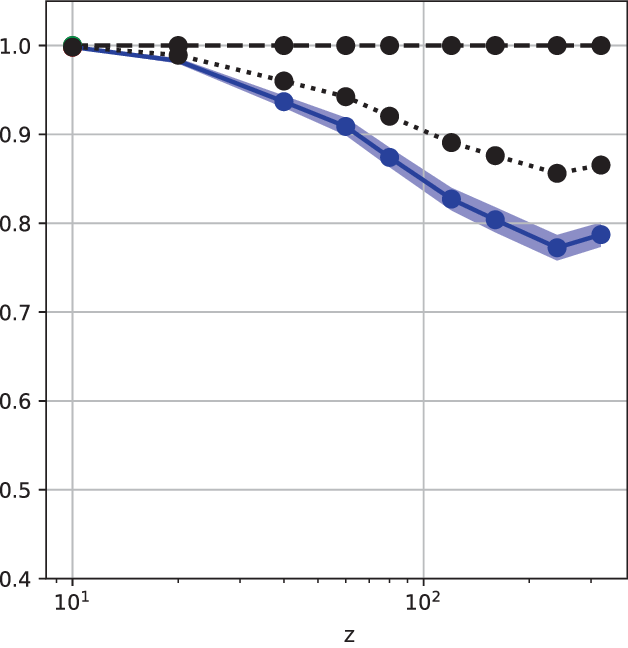}\label{fig:results_node_d}}
\caption{Performance of the assignment algorithm as measured by normalized waiting time, averaged over 500 iterations. The shaded area shows a 95\% confidence interval. The algorithms were evaluated on a $16 \times 16$ grid with a $50$\,m separation. The robot origins are perturbed through added Gaussian noise with $\sigma=100$. We have $N_\mathrm{d}=\{4,6,\ldots,16\}, M=4, N=16$. Goal locations are sampled randomly within the gridmap for each iteration. Waiting times are normalized by the \emph{Hungarian}, and thus show the improvement w.r.t. the initial non-redundant assignment. For clarity, we separate the results into two subfigures: figures (a) and (c) show results for \emph{Greedy}, \emph{Random}, \emph{Best a-posteriori}, and figures (b) and (d) show results for \emph{Optimal} and its derived bound. 
\label{fig:results_node}}
\end{figure}

\subsection{Benchmark Algorithms}
\label{sec:benchmarks}
The performance of our method (\emph{Greedy}) is compared to four alternate assignment algorithms: 
\textbf{(1)} \emph{Hungarian:} We implement the Hungarian method on expected waiting times for a non-redundant assignment of $N_{\mathrm{d}}= M$ robots (i.e., $\mathcal{A}=\emptyset)$. This represents the initial assignment $\mathcal{O}$, and is used as the baseline for all following (redundant) assignment algorithms.
\textbf{(2)} \emph{Random:} A random algorithm assigns the redundant $N_{\mathrm{d}} - M$ robots randomly to goals.
\textbf{(3)} \emph{Repeated Hungarian:} We implement repeated iterations of the Hungarian assignment algorithm (at each iteration, assigning $M$ redundant robots in one go), until the cap $N_{\mathrm{d}} - M$ is reached.
\textbf{(4)} \emph{Best a-posteriori:} This corresponds to the best a-posteriori performance for a given set of robot origins and goal destinations, based on true (observed) travel times, on which we run the Hungarian method with $N_{\mathrm{d}}= M$ robots. 
\textbf{(5)} \emph{Optimal:} We implement an exhaustive search strategy using dynamic programming, which makes ${O}(M 2^N)$ calls to the objective function (assuming an initial assignment $\mathcal{O}$ has already been made)~\footnote{Evaluations show that implementing \emph{Optimal} with a choice of $N$ robots (instead of $N-M$) gives imperceptible performance gains.}.

For each of the simulation series below, we define an underlying noise distribution $\mathcal{D}$ that describes the uncertainty around our travel time estimates. Using this distribution $\mathcal{D}$, we also sample `true' (observed) values, which we use for our performance evaluation. This value is unknown to all algorithms except \emph{Best a-posteriori}.

\subsection{Comparison to Optimal}
\label{sec:optimal}
We perform simulations that compare our algorithm to \emph{Optimal} for $N_\mathrm{d}=\{4,6,\ldots,16\}, M=4, N=16$. For this comparison, we consider uncertainty on robot origin nodes.
For all algorithms except \emph{Best a-posteriori}, we add noise to the origins of the robots. The noise is sampled from a 2D Gaussian, with uncorrelated noise with a standard deviation $\sigma=100$. We add this noise to a given robot origin $r_i$; the noisy position $\tilde{r}_i$ is the node that is closest.
Robot origin and goal locations are randomly positioned on a $16\times16$ regular grid with 50\,m separation. Travels speeds are drawn from a normal distribution with mean 10\,m/s and standard deviation 2\,m/s.

Fig.~\ref{fig:results_node_a} shows the normalized waiting time $J/J_0$, as a function of the maximum deployment size $N_{\mathrm{d}}$. The results show that \emph{Greedy} performs near-optimally and well below the bound, with the mean values coinciding with \emph{Optimal} for all values of $N_\mathrm{d}$. Redundant assignment clearly improves upon non-redundant assignment. Further, for the same deployment size, our algorithm performs significantly better than randomly assigned redundant robots. Fig.~\ref{fig:results_node_b} compares redundant with non-redundant assignment for varying Gaussian noise values. The results show how the improvement of \emph{Greedy} over \emph{Hungarian} increases as the noise increases. For very large noise values (w.r.t. the size of the workspace), the performance of \emph{Random} approaches that of \emph{Greedy}. In summary, these results confirm the benefit of redundant assignments, and demonstrate the near-optimality of the proposed approach.

\subsection{Independent Noise with Node Uncertainty}
\label{sec:results_node}

In this section, we analyze the effect of node position uncertainty (at robot origins), on random transport networks. First, we assume that the noise is independent across the nodes -- this assumption will be relaxed in the following section.
Fig.~\ref{fig:results_rg_node} reports a series of simulations in which we test our algorithm on transport graphs with travel time uncertainty due to uncertain robot positioning. 
We evaluate Algorithm~\ref{alg1} on a set of 500 random undirected connected transport networks with 200 nodes (of which Fig.~\ref{fig:graph_instance_a} shows an example). Our default values are $N=25$ robots, $N_{\mathrm{d}}=20$ robots, $M=5$ goals, $S=200$ samples, $K=1$ path options. The noisy node positions are generated by selecting one node (uniformly at random) out of the set of nearest neighboring nodes based on Euclidean distance, for a given neighborhood size. The default neighborhood size is 4, which means that one out of 5 nodes is randomly sampled. Robots are initially located at 10 randomly selected hubs.

Fig.~\ref{fig:node_rg_a} shows the normalized waiting time as a function of the deployment size. The results show that \emph{Greedy} performs significantly better and \emph{Random} and \emph{Repeated Hungarian}. As the maximum deployment size tends toward the total number of robots, the performance tends toward the ideal setting (\emph{Best a-post.}), and the difference in performance between the algorithm variants decreases.

Fig.~\ref{fig:node_rg_b} shows the normalized waiting time as a function of the positioning noise. \emph{Greedy} performs better than the variant methods, with the performance difference decreasing as the noise increases. This trend is analogous with the trend in Fig.~\ref{fig:results_node_b}.

\begin{figure}[tb]
\centering
\psfrag{a}[lc][][0.7]{Hungarian}
\psfrag{e}[lc][][0.7]{Random}
\psfrag{o}[lc][][0.7]{Repeat. Hung.}
\psfrag{n}[lc][][0.7]{Greedy}
\psfrag{s}[lc][][0.7]{Best a-post.}
\psfrag{x}[cc][][0.7]{$K$ Paths}
\psfrag{m}[cc][][0.7]{Deployment Size $N_{\mathrm{d}}$}
\psfrag{v}[cc][][0.7]{Noise (neighborhood size)}
\psfrag{c}[cc][][0.7][90]{Normalized Waiting Time [s]}
\subfigure[]{\includegraphics[height=4.5cm]{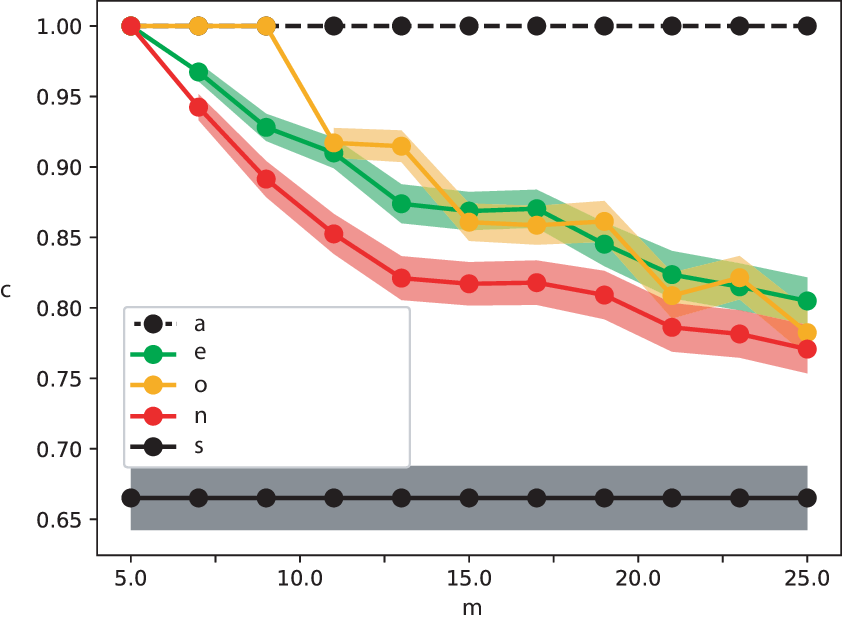}\label{fig:node_rg_a}}\hspace{0.3cm}
\subfigure[]{\includegraphics[height=4.5cm]{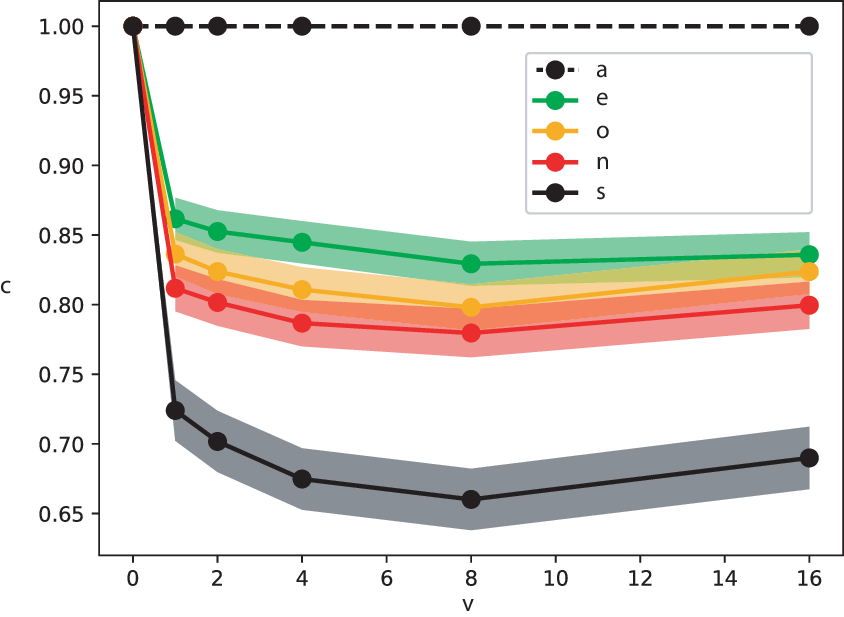}\label{fig:node_rg_b}}
\caption{Performance of the assignment strategies, as measured by normalized waiting time. Each data-point is averaged over 500 runs (random graphs), and the shaded areas show 95\% confidence intervals. (a) Performance as a function of deployment size $N_{\mathrm{d}}$ for $N=25$. (b) Performance as a function of the noise on node positions (implemented through a neighborhood method).
\label{fig:results_rg_node}}
\end{figure}


\subsection{Correlated Noise with Edge Uncertainty}
\label{sec:results_edge}

In this section, we analyze the effect of uncertain edge costs. In other words, the travel time along the edges of the transport network is uncertain.
In particular, we assume that this noise can be correlated across the edges. This assumption reflects real-world settings where travel time uncertainty is affected by causes that are correlated across the transport network (e.g., congestion).

As in Sec.~\ref{sec:results_node}, we evaluate Algorithm~\ref{alg1} on a set of random undirected connected graphs with 200 nodes.
Our default values are $N=25$ robots, $N_{\mathrm{d}}=20$ robots, $M=5$ goals, $S=200$ samples, $K=4$ path options. We generate the $K$ path options by taking the shortest (in average) $K$ paths from $i$ to $j$.
Robots are initially located at 10 randomly selected hubs.
%
The joint distribution of travel times along all $MNK$ possible paths is modeled as a multi-variate Gaussian (truncated at 0) with a mean sampled uniformly at random between 10 and 20. Its covariance matrix is such that the diagonal entries are sampled uniformly between 25 and 100, and the off-diagonal correlation factors are generated using a random lower-triangular matrix corresponding to its Cholesky decomposition. 
This allows us to sample from the underlying distribution $\mathcal{D}$. 

%
\begin{figure}[tb]
\centering
\psfrag{a}[lc][][0.7]{Hungarian}
\psfrag{e}[lc][][0.7]{Random}
\psfrag{o}[lc][][0.7]{Repeat. Hung.}
\psfrag{n}[lc][][0.7]{Greedy}
\psfrag{s}[lc][][0.7]{Best a-post.}
\psfrag{x}[cc][][0.7]{$K$ Paths}
\psfrag{m}[cc][][0.7]{Deployment Size $N_{\mathrm{d}}$}
\psfrag{c}[cc][][0.7][90]{Normalized Waiting Time [s]}
\subfigure[]{\includegraphics[height=4.5cm]{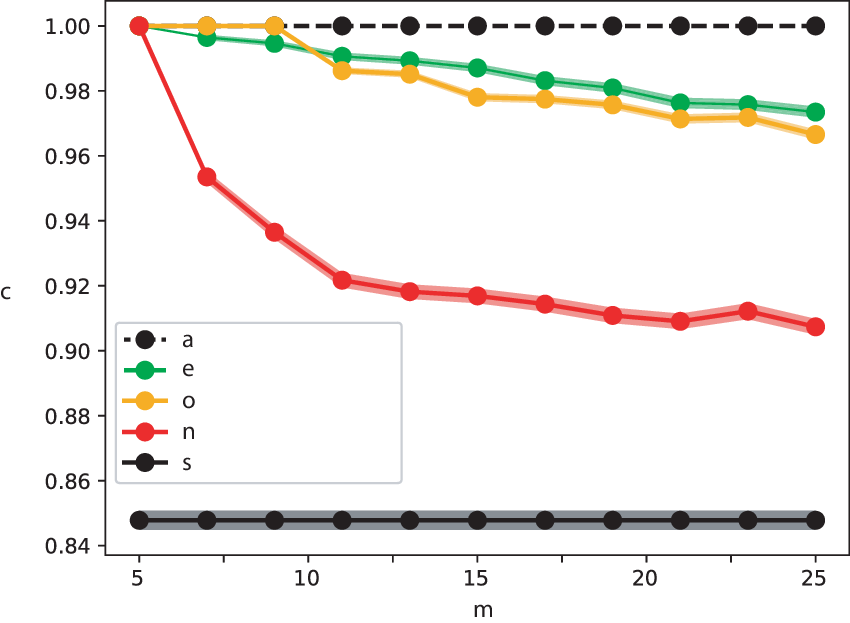}\label{fig:results_edge_a}}\hspace{0.3cm}
\subfigure[]{\includegraphics[height=4.5cm]{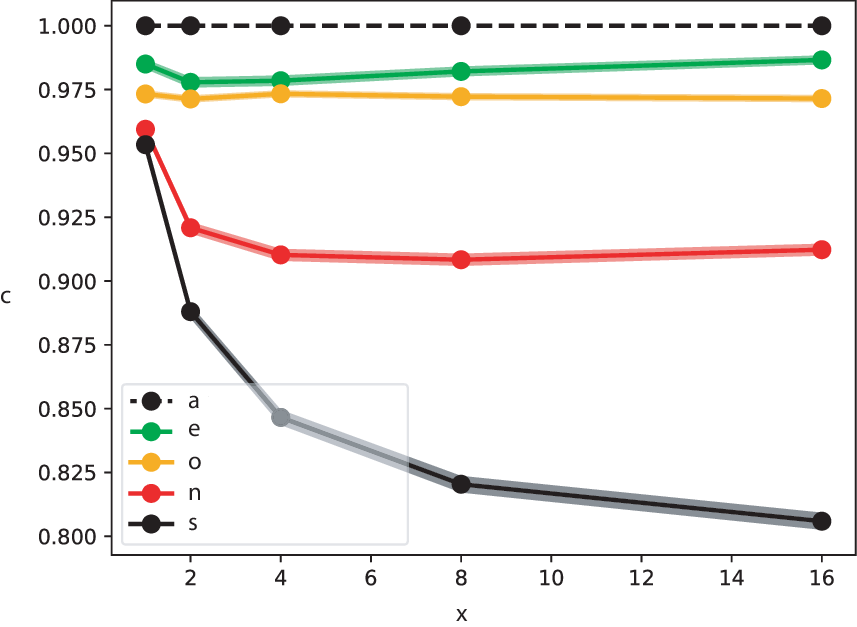}\label{fig:results_edge_b}}
\caption{Performance of the assignment strategies, as measured by normalized waiting time. Each data-point is averaged over 500 runs, and the shaded areas show 95\% confidence intervals. (a) Performance as a function of deployment size $N_{\mathrm{d}}$ for $N=25$. (b) Performance as a function of $K$ paths, $N_{\mathrm{d}} = 20$.
\label{fig:results_edge}}
\end{figure}


Fig.~\ref{fig:results_edge} shows the performance of our algorithm, as measured by the normalized waiting time $J/J_0$. Fig.~\ref{fig:results_edge_a} shows how, as we increase the total robot deployment $N_{\mathrm{d}}$, the waiting time decreases, approaching the lower bound (\emph{Best a-posteriori}). Fig.~\ref{fig:results_edge_b} shows how, as we increase the number of path options $K$ to be considered by the assignment algorithm, performance improves initially, but then flattens out. This validates our usage of a fixed cap ($K$) on the number of paths to be considered by the algorithm. 
We see that any redundant assignment strategy improves upon non-redundant assignment. Our solution \emph{Greedy} improves significantly upon the benchmarks \emph{Random} and \emph{Repeated Hungarian}.

Fig.~\ref{fig:results_diversity} shows the correlation of paths (within robot coalitions) in the solutions found by three strategies (\emph{Greedy}, as well as \emph{Repeated Hungarian} and \emph{Random}). For each robot coalition assigned to one goal, we compute the average pairwise correlation between all pairs of paths found for the robots belonging to that coalition. The latter value is averaged over all coalitions. 
We observe that the correlation of paths within coalitions generated by \emph{Greedy} is lower than that of both \emph{Random} and \emph{Repeated Hungarian}, across all uncertainty distribution correlation values. This indicates that paths selected by \emph{Greedy} tend to be more diverse. 

Figure~\ref{fig:results_prob2} shows results for Problem~\ref{prob:problem2}, where we find a solution to the number of robots needed, $|\mathcal{A}_G|$, as a function of the percent improvement over a non-redundant assignment (Hungarian method). \emph{Greedy} achieves a much higher improvement than \emph{Random} and \emph{Repeated Hungarian}, for the same number of robots deployed.

\begin{figure}[tb]
\centering
\psfrag{a}[lc][][0.7]{Hungarian}
\psfrag{e}[lc][][0.7]{Random}
\psfrag{o}[lc][][0.7]{Repeat. Hung.}
\psfrag{n}[lc][][0.7]{Greedy}
\psfrag{s}[lc][][0.7]{Best a-post.}
\psfrag{c}[cc][][0.7][90]{Path Correlation across Coalitions}
\psfrag{m}[cc][][0.7]{Edge Correlation}
\includegraphics[height=4.5cm]{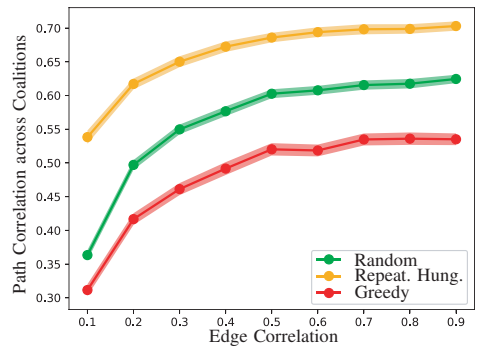}
\caption{Correlation across the paths selected by the robots in redundant coalitions. Each data-point is averaged over 500 runs, over all robot coalitions formed. The shaded areas show a 95\% confidence interval.
\label{fig:results_diversity}}
\end{figure}

\begin{figure}[tb]
\centering
\psfrag{a}[lc][][0.7]{Hungarian}
\psfrag{e}[lc][][0.7]{Random}
\psfrag{o}[lc][][0.7]{Repeat. Hung.}
\psfrag{n}[lc][][0.7]{Greedy}
\psfrag{s}[lc][][0.7]{Best a-post.}
\psfrag{m}[cc][][0.7][90]{Number of robots deployed $|\mathcal{A}_G|$}
\psfrag{i}[cc][][0.7]{\% improvement over Hungarian}
\includegraphics[height=4.5cm]{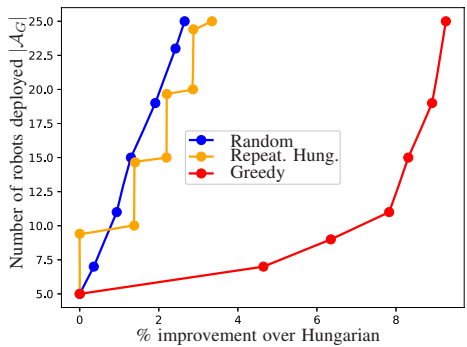}
\caption{The number of robots deployed, $|\mathcal{A}_G|$, as a function of the percent improvement over a non-redundant assignment (Hungarian method). Given a target improvement over the Hungarian method, each data-point represents the average performance across 500 runs on random transport networks.
\label{fig:results_prob2}}
\end{figure}


\section{Discussion} 
\label{sec:discussion}
Although the idea of engineering systems with \emph{redundant resources} to increase reliability, robustness and resilience is not new~\cite{kulturel:2003,ghare:1969}, our ideas in this work provide a different take on the concept of redundancy. In particular, we provide a mathematical framework that allows us to reason about the added value of redundancy for mobile systems where the cost of time comes at the highest premium. Although we do provide an option to limit the number of additional robots to deploy, future work should more explicitly address the trade-off between the cost of providing redundancy (e.g., operational costs, maintenance costs) and performance gains. Within this context, future studies should analyze the economy of on-demand task assignment systems, analyzing how much more users are willing to pay for improved quality of service and reduced waiting times. Richer variants of the problem statement would consider budget constraints as well as heterogeneous robots with different costs. 
A potential limitation of our approach is that we do not explicitly model the redistribution of robots due to the redundant assignment scheme. Whether this robot re-balancing is beneficial remains to be studied.

The results in Sec.~\ref{sec:results_edge} indicate an interesting connection between resilience and diversity: the paths selected by \emph{Greedy} in the redundant robot coalitions tend to be more diverse (and correlate with better performance). This insight is illustrated in Fig.~\ref{fig:results_paths}, which shows two instances of paths selected by a redundant coalition of 5 robots, using \emph{Greedy} in Fig.~\ref{fig:results_paths_a} and \emph{Repeated Hungarian} in Fig.~\ref{fig:results_paths_b}. The \emph{Repeated Hungarian} method sends robots along the same perceived best path, whereas \emph{Greedy} evaluates the added gain for each new redundant robot, and hence, diversifies the selected paths. Although these results provide insights to the algorithm's inner workings, more work needs to be done in order to explicitly exploit the coupling between resilience and diversity. 

Finally, our current work only focuses on static assignment for a given batch of tasks. For task assignment problems with a continuous influx of tasks, our framework would need to be extended. In previous work, we implemented such a mechanism through a sliding window approach that keeps a reserve of robots to accommodate future (unknown) task demands~\cite{prorok:2017iros}. This idea should be further extended to optimize the re-balancing of the robot distribution, as a function of predictions of spatio-temporal demand distributions.

\begin{figure}[tb]
\centering
\psfrag{a}[lc][][0.7]{Hungarian}
\subfigure[]{\includegraphics[width=0.47\columnwidth]{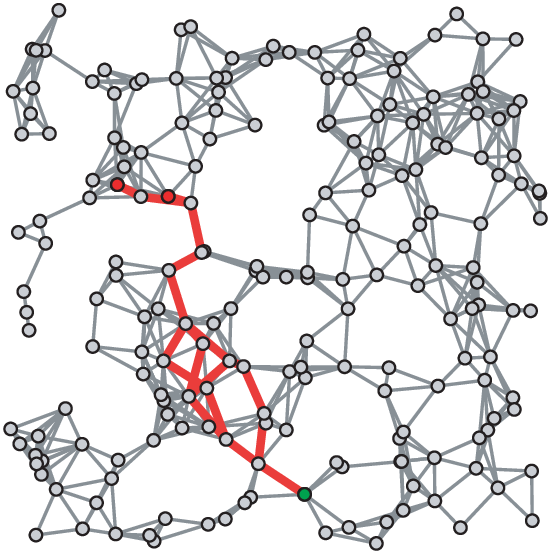}\label{fig:results_paths_a}}\hspace{0.2cm}
\subfigure[]{\includegraphics[width=0.47\columnwidth]{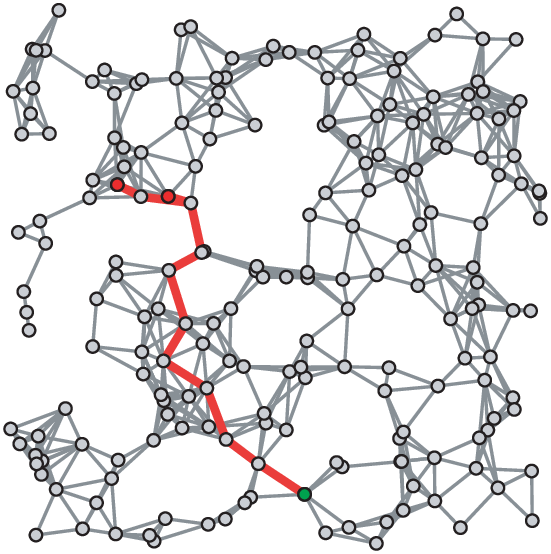}\label{fig:results_paths_b}}
\caption{Paths selected by a robot coalition initially located at two separate hubs (red nodes), and assigned to a goal (green node) for (a) \emph{Greedy} and (b) \emph{Repeated Hungarian}.
\label{fig:results_paths}}
\end{figure}

\section{Conclusion} 
\label{sec:conclusion}
In this work, we provided a framework for robot-to-goal assignment that is resilient to uncertainty over robot travel times. The main novelty is the exploitation of robot redundancy, and the formulation of a supermodular optimization framework that efficiently and near-optimally selects redundant robot matchings to minimize the average waiting time at the goal locations.
Our first-come-first-to-serve principle implies a \emph{minimum} aggregation over redundant assignments. This allows us to compute our objective function efficiently by dynamic programming, leading to a polynomial-time algorithm that can be run in real-time, even for large numbers of robots, goals, and graph nodes.
Our results show that redundant assignment reduces waiting time with respect to non-redundant assignments. This performance gap between redundant and non-redundant assignment increases with increasing noise levels. The proposed redundant assignment algorithm is valid for the general problem of uncertain travel times. In particular, we do not make any explicit assumptions on the underlying uncertainty models.
Finally, our findings include results on the benefit of diversity and complementarity in redundant robot coalitions; these are unprecedented insights within this context, and contribute towards providing resilience to uncertainty through targeted robot team compositions.


\section*{Acknowledgment}
This work is supported by ARL DCIST CRA W911NF-17-2-0181, and by the Centre for Digital Built Britain, under InnovateUK grant number RG96233, for the research project ``Co-Evolving Built Environments and Mobile Autonomy for Future Transport and Mobility''.



\bibliographystyle{abbrvnat}
{\small
\bibliography{Bibliography}
}


\end{document}